\documentclass[lettersize,journal]{IEEEtran}
\usepackage{amsmath,amsfonts}
\usepackage{algorithmic}
\usepackage{algorithm}
\usepackage{array}
\usepackage[caption=false,font=normalsize,labelfont=sf,textfont=sf]{subfig}
\usepackage{textcomp}
\usepackage{stfloats}
\usepackage{url}
\usepackage{verbatim}
\usepackage{graphicx}
\usepackage{cite}
\usepackage{multirow}
\usepackage{colortbl}
\usepackage[table]{xcolor}
\usepackage{amssymb}
\usepackage{amsthm}
\theoremstyle{plain}
\newtheorem{theorem}{Theorem}[section]

\newtheorem{proposition}[theorem]{Proposition}
\begin{document}

\title{MVAnimate: Enhancing Character Animation with \\ Multi-View Optimization}

\author{Tianyu~Sun, Zhoujie Fu, Bang Zhang, Guosheng Lin
\thanks{Tianyu Sun, Zhoujie Fu, and Guosheng Lin are with the College of Computing and Data Science (CCDS), Nanyang Technological University. }
\thanks{Bang Zhang is with Alibaba Group. }}



\maketitle

\begin{abstract}
The demand for realistic and versatile character animation has surged, driven by its wide-ranging applications in various domains. However, the animation generation algorithms modeling human pose with 2D or 3D structures all face various problems, including low-quality output content and training data deficiency, preventing the related algorithms from generating high-quality animation videos. Therefore, we introduce MVAnimate, a novel framework that synthesizes both 2D and 3D information of dynamic figures based on multi-view prior information, to enhance the generated video quality. Our approach leverages multi-view prior information to produce temporally consistent and spatially coherent animation outputs, demonstrating improvements over existing animation methods. Our MVAnimate also optimizes the multi-view videos of the target character, enhancing the video quality from different views. Experimental results on diverse datasets highlight the robustness of our method in handling various motion patterns and appearances. 
\end{abstract}


\section{Introduction}

\begin{figure}[h]
  \centering
  \includegraphics[width=0.4\textwidth]{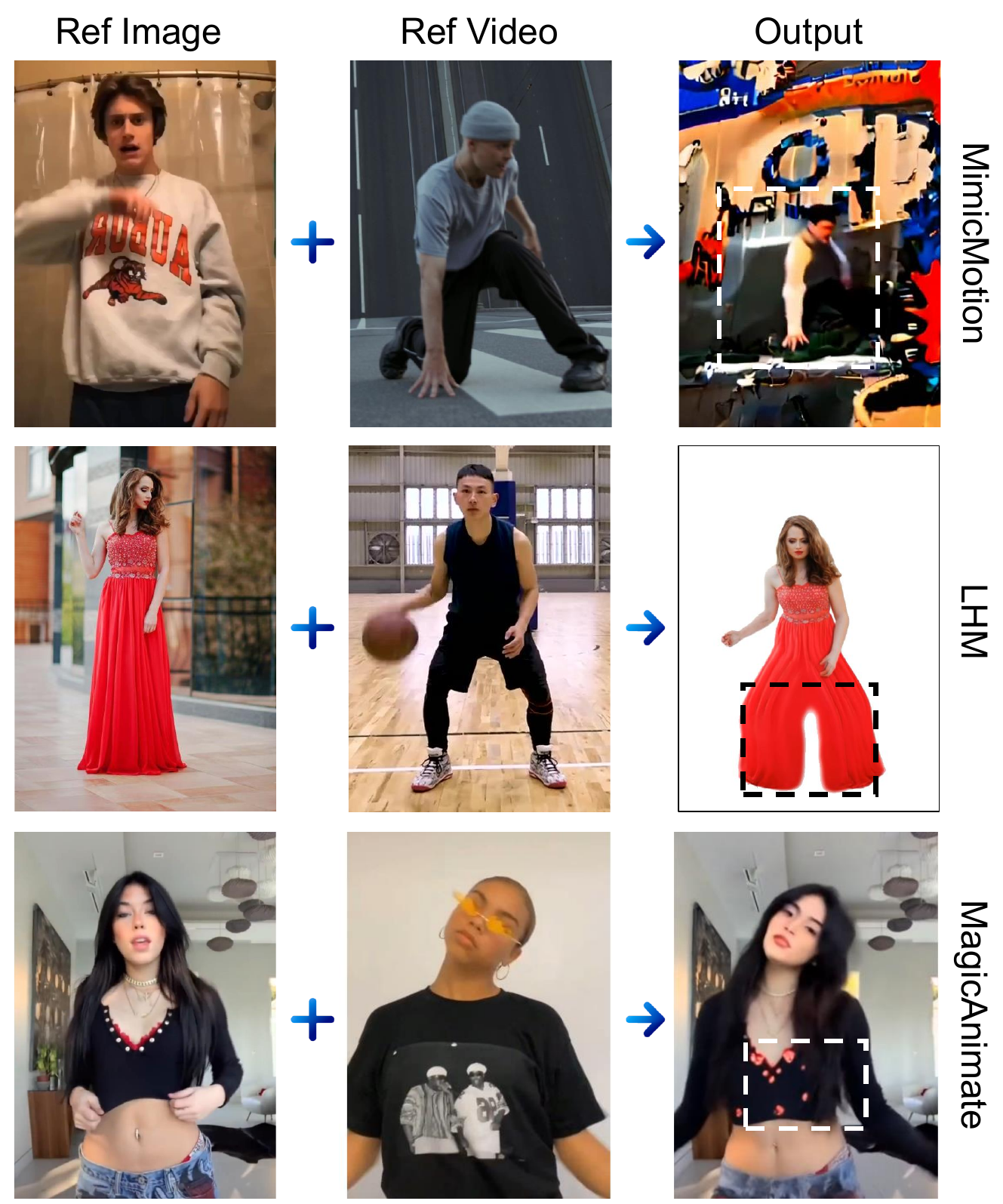}
  \caption{\textbf{Some common problems in related research works.} Here we show three examples of some of the SOTA animation algorithms. In the first row, when the reference video involves complex gestures, it is common for the performance to get relatively poor. In the second row, the texture of the reference video occasionally affects the output video and distorts the generated frame texture. }
  \label{fig:intro}
\end{figure}

\IEEEPARstart{C}{haracter} animation is a crucial area of research with wide-ranging implications for various fields, including virtual reality, video editing, video games, etc~\cite{tu2024stableanimator,ma2024follow,xiang2025remask,zhao2024ta2v,wang2012movie2comics}. By enabling the creation of realistic and dynamic human avatars, this technology enhances immersive experiences and presents new opportunities for creating multimedia content~\cite{yariv2024diverse,hu2023benchmark}. As the demand for high-quality, lifelike human videos grows, research in this domain becomes increasingly vital for driving innovation and meeting the needs of diverse applications.

Despite recent advances, existing character animation algorithms still face two major challenges that limit their generalization and visual quality. 
\textbf{One major problem} is that the model implemented for pose representation remains a significant bottleneck~\cite{xue2024follow,wang2024unianimate}. Many prior works adopt 2D pose estimation models such as OpenPose~\cite{cao2017realtime}, DensePose~\cite{guler2018densepose}, or DWPose~\cite{yang2023effective}, which offer fast and dense annotations. However, unlike other tasks~\cite{sun2023trosd,yin2026dformer++}, these 2D representations often lack sufficient spatial structure or depth cues, making it difficult to preserve the 3D consistency of the animated human figure, especially under complex poses. The first example in Fig.~\ref{fig:intro} shows a result of DWPose-based MimicMotion~\cite{zhang2024mimicmotion}, which fails to acquire the 3D information from the reference video. Some other 3D parametric models, such as SMPL~\cite{loper2015smpl} or SMPL-X~\cite{pavlakos2019expressive}, can provide full-body pose and shape estimations in 3D space, but they represent the human body as a mesh with fixed topology, struggling to accurately model detailed appearance, resulting in unnatural or overly rigid animations, such as the second displayed output of LHM~\cite{qiu2025lhm} with a SMPL-X extractor. These limitations motivate the need for a hybrid pose encoding strategy that can effectively combine the structural robustness of 3D models with the appearance-aware detail of 2D representations. 
\textbf{Another major problem} is that texture distortion remains an unresolved issue~\cite{xue2025human}. In many diffusion-based animation pipelines, the training scheme involves conditioning on both the reference image and reference video~\cite{chang2023magicpose,xu2024magicanimate,zhang2024mimicmotion}. When the training reference image is adopted from the reference video itself, the model may inadvertently learn to entangle the temporal texture patterns from the reference video with the static appearance of the target character, leading to visible artifacts and degraded visual fidelity during inference periods. As shown in the third row of Fig.~\ref{fig:intro}, this operation of MagicAnimate~\cite{xu2024magicanimate} exerts a "contamination" effect that often manifests as unwanted texture transfer, especially when the reference video contains strong shadows, motion blur, or inconsistent lighting.

To address these challenges, we posit that the multi-view prior not only enhances 3D coherence but also regularizes texture consistency across viewpoints. Therefore, we propose MVAnimate, a novel character animation framework based on multi-view optimization. Instead of relying solely on a single modality for pose representation, MVAnimate integrates both 2D and 3D pose features to leverage their respective strengths. We also incorporate multi-view information of the target character using pre-trained view synthesis models, enabling us to build a spatial-temporal encoder that preserves cross-frame coherence. Furthermore, we introduce a dedicated optimization scheme that decouples pose and appearance during training, effectively mitigating texture distortion issues. Finally, to address texture distortion caused by the mostly used training technique in character animation, we introduce a possible solution to augment existing datasets and refine the final results. 

Our pipeline supports both 2D and 3D output video generation and is built upon a diffusion model backbone enhanced with temporal and multi-view awareness. Compared with existing methods~\cite{hu2024animate,xu2024magicanimate,zhang2024mimicmotion}, our approach offers more stable results on challenging sequences, including fast motion, occlusions, and diverse clothing.

In summary, our main contributions are:
\begin{itemize}
\item We propose a multi-view character animation method that combines 2D and 3D pose representations, addressing the limitations of each and improving structural accuracy across diverse human motions.
\item We introduce a training strategy to alleviate texture distortion by decoupling temporal video features from target appearance, resulting in cleaner and more consistent visual outputs.
\item Besides fine 2D video editing function, our method also offers an optimization module to output multi-view videos of the character with higher quality. 
\end{itemize}

\section{Related Works}



\subsection{Character Animation}

In recent years, significant advancements have been made in the field of character animation. Hu et al.~\cite{hu2024animate} introduced Animate Anyone, a framework designed to transform character images into animated videos controlled by desired pose sequences. Chang et al.~\cite{chang2023magicpose} proposed MagicPose, a diffusion-based model for 2D human pose and facial expression retargeting, which maintains the identity of the subject while allowing for realistic pose and expression changes. Zhang et al.~\cite{zhang2024mimicmotion} designed MimicMotion, a high-quality human motion video generation framework that uses confidence-aware pose guidance to achieve temporal smoothness and robustness in generated videos. 

Despite these advancements, a common limitation among these algorithms is their struggle with generating long, temporally coherent videos without artifacts, and the need for extensive training data to achieve high-quality results. Unlike MagicPose~\cite{chang2023magicpose}, which relies on single-view consistency, our approach explicitly optimizes across views. Unlike MimicMotion~\cite{zhang2024mimicmotion}, which assumes uniform confidence across frames, our method learns per-view reliability via adaptive weighting (Proposition~\ref{prop:invvar}).

\begin{figure*}[t]
  \includegraphics[width=\textwidth]{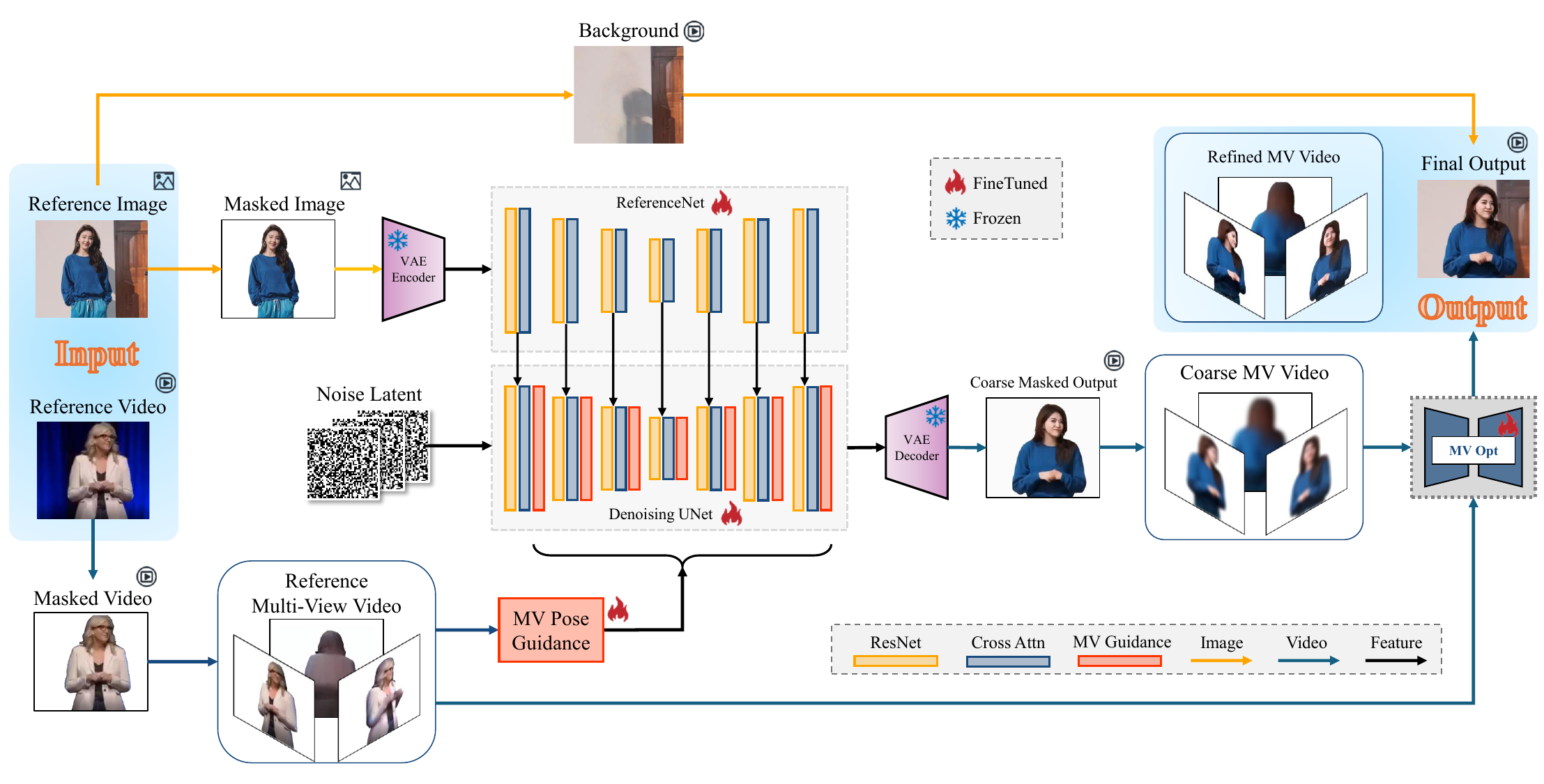}
  \caption{\textbf{Overview of the inference pipeline of our MVAnimate.} Our MVAnimate consists of two stages: 1) multi-view guided coarse video generation, which includes a ReferenceNet branch to extract appearance features from the reference image, and a denoising UNet branch to learn the multi-view pose features from the reference multi-view videos, and 2) multi-view optimization, which refines the coarse video output. We add the legends of the basic modules and data flow on the bottom right. }
  \label{fig:pipeline}
\end{figure*}

\subsection{Multi-View Video Generation}

Recent achievements in multi-view video generation have significantly improved the quality and realism~\cite{yi2024mvgamba,yi2024diffusion}. Xie et al.~\cite{xie2024sv4d} proposed the SV4D algorithm, which generates dynamic 3D content with multi-frame and multi-view consistency. The method simplifies the generation process and enhances the coherence of the resulting 3D content. There have also been a number of 3D human generation pipelines recently, such as SinGS~\cite{Wu2025sings}, HumanRef~\cite{zhang2024humanref}, Human4DiT~\cite{shao2024360}, etc. These works manage to generate novel views from the original one or more views of a human video. These works collectively offer new insights and methodologies for creating high-quality multi-view human motion videos.

\section{Preliminary}

\subsection{Task Setting}

Character animation is mainly a controllable video generation task, with the input of images, videos, and occasionally text, audio, etc. In our case, let $\mathcal{I}$ be the image space and $\mathcal{V}$ the space of $N$-frame videos.
Given a target (reference) image $I_{\mathrm{ref}}\in\mathcal{I}$ and a driving (reference) video $V_{\mathrm{ref}}=[f_1,\dots,f_N]\in\mathcal{V}$, MVAnimate seeks a generator
$G_\theta:\mathcal{I}\times\mathcal{V}\to\mathcal{V}$ such that
$V_\theta=G_\theta(I_{\mathrm{ref}},V_{\mathrm{ref}})$ preserves the identity and background of $I_{\mathrm{ref}}$ while following the pose/motion of $V_{\mathrm{ref}}$.
A pretrained multi-view synthesizer $S$ (SV4D 2.0) provides $M$ coarse novel-view videos
$\{V^{(m)}_{\mathrm{ref}}\}_{m=1}^M=S(V_{\mathrm{ref}})$ used as guidance features.
Denote pose extractor $\Phi$ (DWPose) and feature encoder $\Psi$. As for the background of the output, we keep it consistent with that of the reference image $I_{ref}$. 

In our work, we adopt SV4D 2.0~\cite{yao2025sv4d}, a pre-trained 3D video generation model, to generate a coarse multi-view video. This multi-view video is first implemented in the spatiotemporal guidance encoder to train the diffusion model. This coarse multi-view video is then optimized in the multi-view optimization module to enhance the quality of the 2D video output and the multi-view video output. 

\subsection{Diffusion Model Training}

Diffusion models (DM) in video generation tasks are often implemented as the backbone network, so we follow suit in our image encoding module. We adopt the vanilla variational autoencoder~\cite{kingma2013auto} to map the reference image to the latent space. Based on the unified formulation of the Denoising Diffusion Probabilistic Model~\cite{ho2020denoising}, the input $I_{ref}$ adds $t$ step of noise to get $I_t$. DM also introduces Gaussian latent noise $\mathcal N(\textbf{0},\textbf{I})$ to the latent space feature $z_0$ and produces a noisy latent $z_t$ at step $t$. In this case, we can denote the loss of our diffusion model as: 

\begin{equation}
\label{eq:dm}
\begin{split}
\begin{aligned}
    L=\mathbf{\mathbb{E}_{I_t, \mathcal G,\mathcal N(\textbf{0},\textbf{I}),t}||\epsilon-\epsilon_\theta(I_t|\mathcal G,t)||_2^2},
\end{aligned}
\end{split}
\end{equation}

where $\mathcal G$ denotes the multi-view guidance from multi-view videos $V_{ref}$, and $\epsilon_\theta$ marks a denoising function parameterized with parameter $\theta$. This will be the training objective of the diffusion backbone of our method. 

Furthermore, to justify the introduction of multi-view guidance $\mathcal G$, we need to examine the statistical optimality of conditioning on $\mathcal G$ based on Proposition~\ref{prep:score}. Note that for the propositions and theorems introduced in our paper, we provide the missing proofs in the supplementary material. 

\begin{proposition}[Optimal denoiser as conditional score]
\label{prep:score}
Let $z_t=z_0+\sigma_t \epsilon$ with $\epsilon\sim\mathcal{N}(0,I)$ and condition $\mathcal G$ given by multi-view features.
The minimizer $\epsilon_\theta^\star=\arg\min_\theta \mathbb{E}\|\epsilon - \epsilon_\theta(z_t,\mathcal G,t)\|_2^2$ satisfies $\epsilon_\theta^\star(z_t,\mathcal G,t)=\mathbb{E}[\epsilon\mid z_t,\mathcal G,t]$,
equivalently a scaled conditional score $\propto \nabla_{z_t}\log p(z_t\mid \mathcal G,t)$.
\end{proposition}

In this case, the orthogonality of the implemented MMSE estimators under the squared loss leads to Proposition~\ref{prep:score}, which proves that our pose injection at every layer improves the conditional score estimate.

\section{Methodology}
\label{sec:sim}

In this section, we elaborate on our MVAnimation and the crucial modules, including the Multi-View Pose Guidance Network and Multi-View Optimization Module. 

\subsection{Pipeline Overview}

Figure~\ref{fig:pipeline} shows the pipeline of our MVAnimation. The objective of our method is to output a high-quality video of the character from the reference image to dance in the same pose sequence as the reference video. We employ a vanilla LDM-based backbone, adopting the ReferenceNet structure~\cite{zhang2023adding} to guide the diffusion and denoising process, following most generation methods~\cite{liu2023conditional,zhang2024mimicmotion,tu2024stableanimator} for human character animation. The VAE encoder operates on the reference images and generates the latent feature, which interacts with the Denoising UNet. To adapt to video tasks, we inflate the 2D U-Net~\cite{xu2024magicanimate} for image-related tasks to a 3D U-Net. We add temporal attention layers to the Denoising U-Net by implementing self-attention in the time dimension. This will enhance the temporal consistency of the generated video. For the training process, to reduce data requirements and computational overhead, we build upon the pre-trained Stable Video Diffusion (SVD), an open-source image-to-video diffusion model trained on a large-scale video dataset, which demonstrates strong performance in both video quality and diversity compared to contemporary methods. The architecture of MVAnimate is designed to incorporate SVD as a backbone and fine-tune its parameters based on the LoRA scheme~\cite{hu2022lora}.

Following most of the pipelines requiring segmentation~\cite{zhou2023propainter,li2025segment}, we apply Segment Anything Model 2 (SAM2)~\cite{ravi2024sam} to separate the character and the surrounding background. Additionally, to maintain a stable background in the output video, we adopt the Stable Diffusion v2 inpainting model to obtain the background texture from the reference image. Considering that the camera views might be dynamic, we apply patch matching to the inpainted background to ensure the necessary movement of the background in the final output.

\subsection{Multi-View Pose Guidance Network}

\paragraph{Motivation for Adaptive View Weighting}

Before introducing our attention weighting formulation, we first establish a theoretical foundation for why different views should contribute unequally during feature fusion. To justify our motivation, we form the attention block as a kernel regression process. Let $q_t\in\mathbb{R}^d$ be the query token from the main-view latent at step $t$, and $\{k^{(m)}_t,v^{(m)}_t\}_{m=1}^M$ the (key,value) pairs from $M$ guided views. Our MV attention outputs
\begin{equation}
\label{eq:mvopt}
\begin{split}
\begin{aligned}
    \hat{v}_t = \sum_{m=1}^M \alpha^{(m)}_t v^{(m)}_t,\quad
    \alpha^{(m)}_t=\frac{e^{\omega^{(m)}\langle q_t,k^{(m)}_t\rangle}}
    {\sum_{j=1}^M e^{\omega^{(j)}\langle q_t,k^{(j)}_t\rangle}}\!,
\end{aligned}
\end{split}
\end{equation}
where $\omega^{(m)}\!>\!0$ are learned per-view importance weights (Sec.~\ref{sec:mvopt}).
Interpreting $v^{(m)}_t$ as noisy measurements of a latent “canonical” feature $v^\star_t$ with $\operatorname{Var}[v^{(m)}_t]\!=\!\sigma_m^2 I$, the following holds.

Under this setting, each view’s latent embedding can be regarded as a noisy observation of an underlying “true” canonical feature, with varying reliability depending on factors such as visibility and occlusion. To quantify this intuition, we derive an optimal estimator that minimizes the reconstruction error under a Gaussian noise assumption. This analysis leads to Proposition~\ref{prop:invvar}, which shows that the optimal fusion weights are inversely proportional to each view’s noise variance—thereby providing a principled justification for the adaptive weighting strategy used in our multi-view attention module.


\begin{proposition}[Inverse-variance optimality of view weighting]
\label{prop:invvar}
Let $v_t^\star \in \mathbb{R}^d$ be the latent canonical feature and suppose we observe
$v_t^{(m)} = v_t^\star + \varepsilon_t^{(m)}$ for $m=1,\dots,M$, where the noise vectors
$\varepsilon_t^{(m)}$ are independent and distributed as
$\varepsilon_t^{(m)} \sim \mathcal{N}(0,\sigma_m^2 I_d)$ with $\sigma_m^2 > 0$.
Consider linear estimators of the form
\[
\hat{v}_t = \sum_{m=1}^M \beta_m v_t^{(m)} \quad \text{subject to} \quad \sum_{m=1}^M \beta_m = 1.
\]
Then the coefficients minimizing the mean squared error
$\mathbb{E}\|\hat{v}_t - v_t^\star\|_2^2$ are
\[
\beta_m = \frac{\sigma_m^{-2}}{\sum_{j=1}^M \sigma_j^{-2}}, \quad m=1,\dots,M.
\]
\end{proposition}


In practice, our network does not know $\sigma_m^2$ explicitly. 
However, the attention score $\omega^{(m)} \langle q_t, k_t^{(m)} \rangle$ increases for reliable views (high similarity), 
so the softmax over these scores acts as a learned proxy for the inverse-variance rule in Proposition~\ref{prop:invvar}. Hence, the learned attention can be viewed as a differentiable inverse-variance estimator.

\paragraph{Network Details}

On account of a need to acquire prior multi-view information from the monocular input, we utilize SV4D 2.0~\cite{yao2025sv4d} to generate \textbf{coarse} multi-view videos of the character from the reference video and to train the MV Pose Guidance Network with them. We select SV4D 2.0 instead of its previous version due to the clear boundary it generates, improving the pose detection accuracy for our MV-Opt module, which will be further explained later. The view number is empirically set to 8 so that we can obtain a fine geometric relation between these views. 

In most recent works~\cite{xu2024magicanimate,zhou2022magicvideo}, spatial and temporal features extracted from animation videos have been implemented to maintain unified frame and pixel-wise characteristics. When considering the input of multi-view character videos, the spatial and temporal consistency should be measured from different views~\cite{yi2024video}. Moreover, the traditional spatial consistency applied mainly in the 2D domain is ``upgraded'' to 3D due to the geometric prior brought by the inter-relation of the multi-view video frames. 

\begin{figure}
  \includegraphics[width=0.4\textwidth]{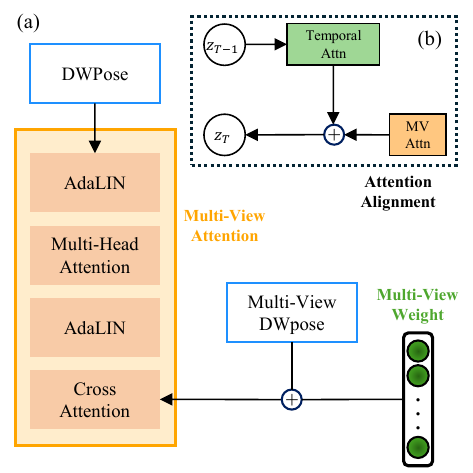}
  \centering
  \caption{\textbf{Attention Alignment in Denoising U-Net.} The module aligns temporal attention and MV attention for the reference videos. (a) depicts the structure of the multi-view attention scheme in the Multi-View Pose Guidance Network, and (b) is an overview of the attention alignment in Denoising U-Net. }
  \label{fig:block}
\end{figure}

Therefore, we conduct a multi-view attention block and align the MV attention with temporal attention. Multi-view attention is shown in Fig.~\ref{fig:block} (a). Inspired by spatial attention~\cite{zhou2022magicvideo,dosovitskiy2020image}, our target is to keep consistency in video texture. We apply multi-head attention to compute a weighted sum of the input, followed by a cross-attention model to compute the frame token $z_{t-1}$ from the main view with the multi-view features. The multi-view videos are given a multi-view weight and then computed into this MV attention block. In the middle of the Multi-View Attention block, we utilize AdaLIN~\cite{adalin} to implement normalization, which is widely implemented in style transfer~\cite{pang2021image} and content generation~\cite{sun2024diffusion} tasks, to replace the more common Linear Normalization to ensure the attention method is robust to shape or texture variation. 

\paragraph{Theoretical Justification for Attention Alignment}

While the additive alignment between temporal and multi-view attentions (Fig.~\ref{fig:block}~(b)) appears heuristic, it can be interpreted as a principled optimization process. We justify this intuition by introducing the following setting and proposition.

Let $A_{\mathrm{temp}}$ and $A_{\mathrm{mv}}$ be the (row-stochastic) temporal and multi-view attention maps at a given block.
Define the latent update as $z_{t-1} = (I - \eta \nabla \mathcal{E})(z_t)$ that performs one gradient step on

\begin{equation}
\label{eq:dm}
\begin{split}
\begin{aligned}
\mathcal{E}(z)=\tfrac{1}{2}\|z - A_{\mathrm{temp}} z\|_2^2 + \tfrac{\lambda}{2}\|z - A_{\mathrm{mv}} z\|_2^2.
\end{aligned}
\end{split}
\end{equation}

Then the additive alignment $A_{\mathrm{joint}}=\tfrac{1}{1+\lambda}(A_{\mathrm{temp}} + \lambda A_{\mathrm{mv}})$ is the first-order optimality condition for the above proximal step. We introduce a proposition to examine the convexity of the optimization process. 

\begin{proposition}[Convex surrogate]
\label{prop:convex}
If $A_{\mathrm{temp}}$ and $A_{\mathrm{mv}}$ are nonexpansive (spectral norms $\le 1$), 
then $\mathcal{E}$ is convex and the additive alignment is a nonexpansive averaged operator.
\end{proposition}

Proposition~\ref{prop:convex} theoretically justifies the additive attention alignment used in our denoising U-Net, showing that it minimizes a convex joint energy and ensures stable, non-expansive updates in feature space.


\subsection{Multi-View Optimization Module}
\label{sec:mvopt}

With the coarse 2D animation result, we implement our multi-view optimization method to generate a refined multi-view animation video of the character. This helps us obtain 3D animation videos with higher quality. Therefore, this module is trained independently of the previous modules and then implemented in the inference period. In the optimization process, we first refine the main (front) view video frame $f_t$ with the timestamp $t$, then optimize the novel views at the same timestamp from front to back, and then move to the frames with timestamp $t+1$. For each timestamp, the optimization operation is implemented 8 times before moving on to the next timestamp. 

Considering that the novel views with smaller angles between the original ones share more pixels in common~\cite{wu2020multi,10542122}, we choose to maintain a tensor $\omega_{mv}$ for the multiple views as the input for the optimization process, and the views with smaller angles have larger weights on the semantic difference evaluation. The initial weights of front views are set higher, where $\omega_{mv}$ is initiated with $[0.25,0.25,0.1,..,0.1]$, and the initial value of the closest views shares the highest values. $\omega_{mv}$ is trained throughout the denoising steps, and then frozen to be implemented in the MV-Opt module afterwards. The elements of this tensor form a mask of all pixels that have a weight share on the original view. 

In the optimization process, to balance the multi-view and temporal prior restrictions, we design a loss function integrating these features:

\begin{equation}
\label{eq:loss}
\begin{split}
\begin{aligned}
    \mathcal L=\omega_1*\mathcal L_{temp}+\omega_2*\mathcal L_{mvP}+\omega_3*\mathcal L_{mvS}.
\end{aligned}
\end{split}
\end{equation}

The loss function consists of three sub-losses. Temporal loss $\mathcal L_{temp}$ computes the MSE loss on DWPose difference between the current frame $f_t$ and the previous frame $f_{t-1}$:

\[
\mathcal L_{temp} = ||DWPose(f_t) - DWPose(f_{t-1})||_2^2.
\]

Multi-view pose loss $\mathcal L_{mvP}$ computes the MSE loss on the corresponding DWPose keypoints vertically between the current frame $f_t$ and the previous frame $f_{t-1}$. For this element, we offer the sub-loss from each view with the same weight:

\[
\mathcal L_{mvP} = \sum_m ||DWPose(f_t^{(m)}) - DWPose(f_{t-1}^{(m)})||_2^2.
\]

The multi-view semantic loss $\mathcal L_{mvS}$ computes the multi-view pixel-wise difference in multi-view frames and adds up under different weights $\omega_{mv}$:

\[
\mathcal L_{mvS} = \sum_m ||f_t^{(m)} - f_t||_2^2.
\]

The coefficients $\omega_1$, $\omega_2$, $\omega_3$ balance temporal smoothness, pose coherence, and semantic similarity, respectively, and are tuned empirically ($\omega_1$ = 1, $\omega_2$ = 0.1, $\omega_3$ = 0.02). Moreover, we need to guarantee the justification of the optimization process with this loss design. Therefore, we construct a Theorem~\ref{thm:monotone} to discuss the descent feature of the MV-Opt minimization. 

\begin{theorem}[Monotone descent of MV-Opt (block coordinate minimization)]
\label{thm:monotone}
Let 
$F(V) = \omega_1 \mathcal{L}_{\text{temp}}(V) + \omega_2 \mathcal{L}_{\text{mvP}}(V) + \omega_3 \mathcal{L}_{\text{mvS}}(V)$
be the multi-view optimization objective in Sec.~IV-C, where each loss is nonnegative and continuously differentiable. 
Consider the cyclic MV-Opt schedule that (i) fixes all frames except $t$ and updates $f_t$, and (ii) within $f_t$ fixes all views except $m$ and updates $f_t^{(m)}$, sweeping $m=1,\dots,M$, and then proceeds to $t+1$.
Assume that every block update produces a sufficient decrease, i.e.
\[
F(x^{k+1}) \le F(x^{k}) - c \|x^{k+1}-x^{k}\|^2
\]
for some $c>0$, and that $F$ has bounded level sets.
Then the sequence $\{F(x^{k})\}$ is nonincreasing and convergent, and every limit point of $\{x^{k}\}$ is a block-coordinate stationary point of $F$.
\end{theorem}

This result guarantees that the alternating updates in the MV-Opt module yield a stable and monotonically decreasing objective, ensuring convergence without oscillation. It provides theoretical support for the consistent improvement observed during multi-view refinement in MVAnimate.

\section{Experimental Results}

In this section, we introduce the experimental results of our algorithm. We also elaborate on the ablation studies of our method based on the results on several datasets. Some more 2D and multi-view video results are disclosed in the video supplementary material. 

\begin{table*}[t]
    \centering
    \caption{\textbf{2D Character Animation Results on TikTok Dataset.} \textcolor{red}{Red} marks the best algorithm, and \textcolor{blue}{blue} stands for the second-best algorithm. }
    \begin{tabular}{l|cccc|cc}
        \hline
        \multirow{2}{*}{Method} & \multicolumn{4}{c}{\cellcolor{lightgray!30}Image} & \multicolumn{2}{|c}{\cellcolor{lightgray!30}Video}\\
        \cline{2-7}
        ~ & PSNR $\uparrow$ & SSIM $\uparrow$ & LPIPS $\downarrow$ & FID $\downarrow$ & FID-VID $\downarrow$ & FVD $\downarrow$\\
        \hline
        DreamPose~\cite{karras2023dreampose} & 28.01 & 0.509 & 0.242 & 38.81 & 78.77 & 379.57\\
        MimicMotion~\cite{zhang2024mimicmotion} & 28.97 & 0.747 & 0.240 & 33.62 & 40.21 & 153.72\\
        DisCo~\cite{wang2024disco} & 29.09 & 0.674 & 0.285 & 28.31 & 55.17 & 267.75\\
        MagicAnimate~\cite{xu2024magicanimate} & 29.16 & 0.714 & \textbf{\textcolor{red}{0.239}} & 32.09 & \textcolor{blue}{39.76} & 147.09\\
        Animate-x~\cite{tan2024animate} & 29.39 & 0.730 & 0.242 & \textcolor{blue}{27.16} & 40.26 & \textcolor{blue}{139.05}\\
        MagicPose~\cite{chang2023magicpose} & \textcolor{blue}{29.53} & \textcolor{blue}{0.752} & 0.292 & \textbf{\textcolor{red}{25.50}} & 46.30 & 165.71\\
         \hline
        MVAnimate & \textbf{\textcolor{red}{29.96}} & \textbf{\textcolor{red}{0.775}} & \textcolor{blue}{0.240} & 30.16 & \textbf{\textcolor{red}{39.11}} & \textbf{\textcolor{red}{132.90}}\\
         \hline
    \end{tabular}
    \label{tab:tiktok}
\end{table*}

\begin{figure*}[t]
  \centering
  \includegraphics[width=0.97\textwidth]{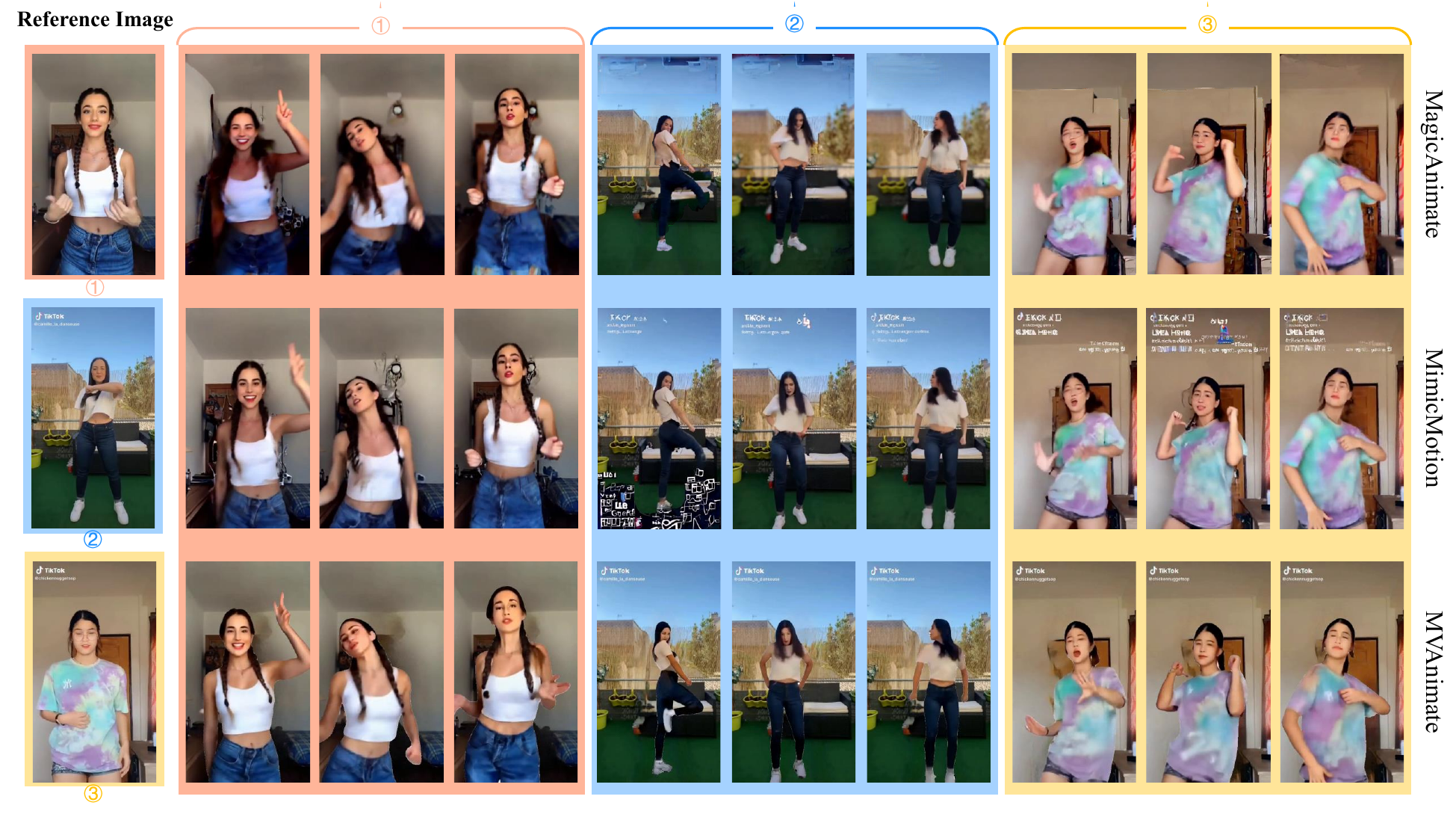}
  \caption{\textbf{Qualitative results on TikTok dancing dataset.} We compare our method with two other SOTA character animation algorithms on different reference images and video frames from the TikTok dataset. }
  \label{fig:tiktok}
\end{figure*}

\subsection{Implementation Details}

The experiments are carried out on the Ubuntu platform, with 8 NVIDIA Tesla V100 for training and a single NVIDIA RTX 6000 Ada for inference. 
As for the augmentation process, we apply virtual try-on algorithms to generate training video pairs in order to alleviate possible texture distortion. 

In most previous works, researchers tend to train each video clip with the first frame as a reference image. Although they make sure the data for training, testing, and evaluation has no intersection, the widely used TikTok~\cite{tiktok} and TED-talks~\cite{tedtalk} datasets have too few samples to train a generative model. Most researchers tend to collect new video datasets from the Internet~\cite{hu2024animate,zhang2024mimicmotion,xu2024magicanimate}, but this would bring an enormous workload for data labeling or calibration. Therefore, we find a simple yet effective way to generate large-scale training datasets by augmenting the original dataset with virtual try-on~\cite{choi2025improving} methods, which can address the problem of data deficiency to some extent. 

The main metrics we utilize measure the quality of the output from the dimension of the image(video frame) and video, including PSNR, SSIM~\cite{wang2004image}, LPIPS~\cite{zhang2018unreasonable}, FID~\cite{heusel2017gans}, FID-VID~\cite{balaji2019conditional}, and FVD~\cite{unterthiner2018towards}. Additionally, more details on the augmentation method and examples of the TikTok dataset for training are disclosed in the supplementary materials. 

\begin{figure*}[t]
  \centering
  \includegraphics[width=0.97\textwidth]{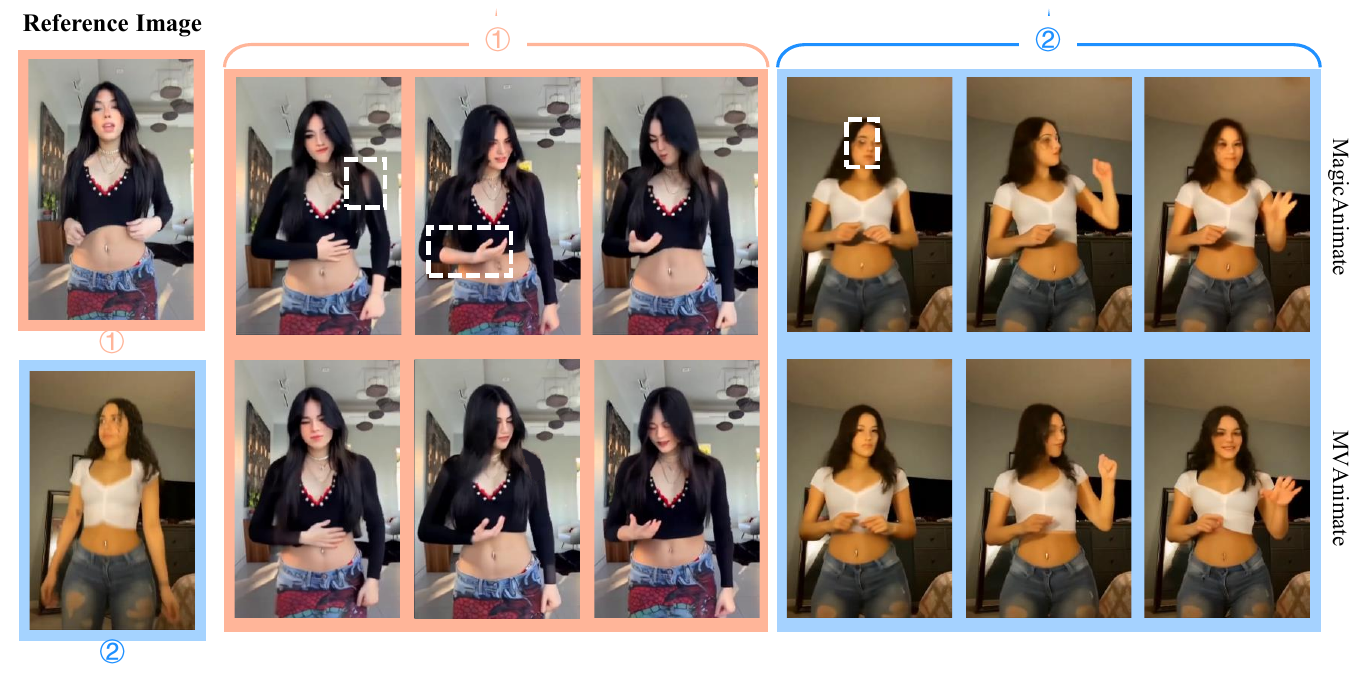}
  \caption{\textbf{Qualitative results on TED-talks dataset.} We compare our method with two other SOTA character animation algorithms on different reference images and video frames from the TED-talks dataset. }
  \label{fig:tedtalk}
\end{figure*}

\subsection{2D Animation Results on TikTok Dataset}

In this section, we present the 2D animation results of our method on the TikTok dancing dataset~\cite{tiktok}. The quantitative results are shown in Tab.~\ref{tab:tiktok}, showing that our MVAnimate has better results compared to other SOTA methods. In Fig.~\ref{fig:tiktok}, we show some qualitative results compared with the mainstream methods MagicAnimate~\cite{xu2024magicanimate} and MimicMotion~\cite{zhang2024mimicmotion}. Our MVAnimate shows a better performance in the three reference images and videos. In the marked areas, our MVAnimate generates clearer motions, gestures, and stable backgrounds. 

\begin{figure*}[t]
  \centering
  \includegraphics[width=\textwidth]{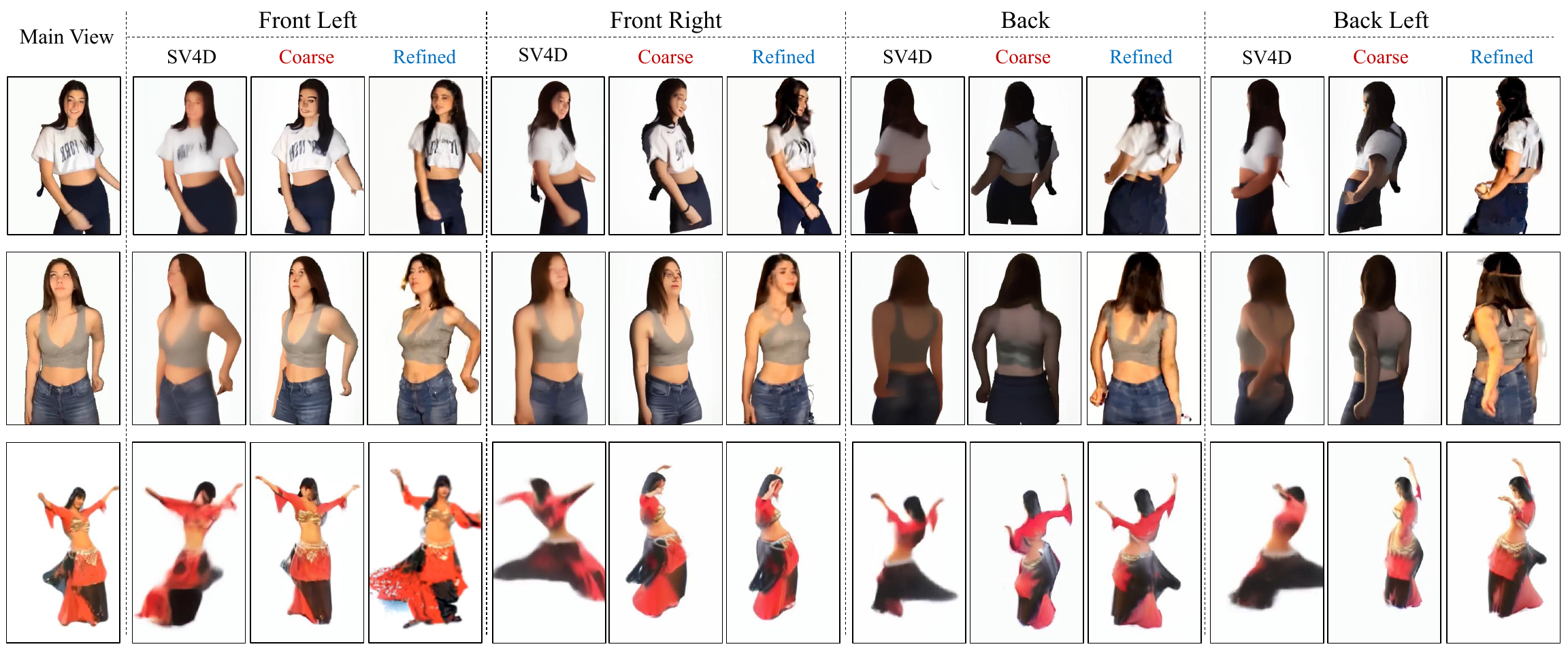}
  \caption{\textbf{Multi-View Video Generation.} Here we show some example frames of the output multi-view videos from our MV-Opt module. The refined outputs show a higher quality compared with the results of SV4D and SV4D 2.0. }
  \label{fig:3d}
\end{figure*}

\begin{table}
    \centering
    \caption{\textbf{2D Character Animation Results on TED-talks Dataset.} The best and second-best metrics are in \textcolor{red}{red} and \textcolor{blue}{blue}. }
    \begin{tabular}{l|c|cc}
        \hline
        \multirow{2}{*}{Method} & \cellcolor{lightgray!30}Video & \multicolumn{2}{c}{\cellcolor{lightgray!30}Image}\\
        \cline{2-4}
        ~ & FVD $\downarrow$ &  FID $\downarrow$ & PSNR $\uparrow$ \\
        \hline
        MimicMotion~\cite{zhang2024mimicmotion} & 145.67 & 26.80 & 30.5\\
        DreamPose~\cite{karras2023dreampose} & 140.12 & 24.80 & \textbf{\textcolor{red}{30.8}}\\
        HumanVid~\cite{wang2024humanvid} & 138.90 & 27.50 & 29.8\\
        MagicPose~\cite{chang2023magicpose} & 136.19 & 25.72 & 30.1\\
        MagicAnimate~\cite{xu2024magicanimate} & \textcolor{blue}{131.51} & \textcolor{blue}{22.78} & 30.5\\
         \hline
        MVAnimate & \textbf{\textcolor{red}{129.66}} & \textbf{\textcolor{red}{22.01}} & \textcolor{blue}{30.7}\\
         \hline
    \end{tabular}
    \label{tab:tedtalk}
\end{table}

\subsection{2D Animation Result on Ted-talks Dataset}

In this section, we present the 2D animation results of our method on the TED-talks dataset~\cite{tedtalk}. Since TED-talk videos contain limited camera motion, methods designed for dance-style datasets (e.g., DisCo~\cite{wang2024disco}, Animate-X~\cite{tan2024animate}) were omitted due to incompatibility with frontal-speaking sequences. We report representative baselines (MimicMotion~\cite{zhang2024mimicmotion}, MagicAnimate~\cite{xu2024magicanimate}) and metrics to ensure fair comparison. 

The qualitative results are shown in Tab.~\ref{tab:tedtalk}. Our MVAnimate achieves the best metrics for FVD, FID, and is second-best for PSNR. In Fig.~\ref{fig:tedtalk}, we show some qualitative results compared with the mainstream methods Animate Anyone and MagicAnimate. Our MVAnimate performs better in the example reference images and videos. The marked regions show that the character and background sequences of our results are more consistent.

\subsection{Multi-View Video Optimization}
\label{sec:mvresult}

In the multi-view optimization module, besides enhancing the quality of the 2D character animation video, the multi-view animation videos are also refined through the optimization process, generating higher-quality multi-view animation videos of the character. In this section, we make a comparison between the coarse multiview videos and the refined ones, together with the first version of the SV4D model~\cite{xie2024sv4d}. Some example multi-view frames are shown in Fig.~\ref{fig:3d}, where we pick 4 out of 7 novel views for display, with the main views listed aside. 

The results unveil some interesting facts, which prove the effectiveness of our model. The novel views from the front are more exquisite than those from the back, showing more details on facial expression, pose position, and clothing textures. It indicates that the views closer to the main view gain more advantage from the optimization process and, as a result, have larger weights during the optimization process. Views from the back, although not as detailed as the front ones, also show an obvious enhancement compared with the coarse ones. Even the darker lighting condition caused by the biased prior information has been corrected by the optimization process, which is quite obvious in the examples of the first two rows. 

\begin{table*}[ht]
    \centering
    \caption{\textbf{Ablation Study on Multi-View Optimization Loss Design.} We test our MV-Opt module in different loss function conditions on the TikTok dancing dataset. The best metrics are marked in \textbf{bold}. }
    \begin{tabular}{ccc|cccc|cc}
        \hline
        \multirow{2}{*}{$L_{temp}$} & \multirow{2}{*}{$L_{mvP}$} & \multirow{2}{*}{$L_{mvS}$} & \multicolumn{4}{c}{\cellcolor{lightgray!20}Image} & \multicolumn{2}{|c}{\cellcolor{lightgray!20}Video}\\
        \cline{4-9}
        ~ & ~ & ~ & PSNR $\uparrow$ & SSIM $\uparrow$ & LPIPS $\downarrow$ & FID $\downarrow$ & FID-VID $\downarrow$ & FVD $\downarrow$\\
        \hline
        \checkmark &  &  & 28.16 & 0.722 & 0.255 & 32.89 & 39.51 & 151.67\\
        \checkmark & \checkmark &  & 29.31 & 0.752 & 0.253 & 31.33 & 39.33 & 143.59\\
        \checkmark & \checkmark & \checkmark & \textbf{29.96} & \textbf{0.775} & \textbf{0.240} & \textbf{30.16} & \textbf{39.11} & \textbf{132.90}\\
         \hline
    \end{tabular}
    \label{tab:loss}
\end{table*}

\begin{table*}[t]
    \centering
    \caption{\textbf{Ablation Study on Multi-View Attention.} We compare attention blocks under different settings to substitute our MV-Attn. The best metrics are marked in \textbf{bold}. }
    \begin{tabular}{c|cccc}
        \hline
        Attention Block Setting & FID-VID $\downarrow$ & FVD $\downarrow$ & PSNR $\uparrow$ & SSIM $\uparrow$\\
        \hline
        LN+MHA+LN+CA & 42.79 (\textcolor{red}{+3.68})& 137.41 (\textcolor{red}{+4.51})& 28.61 (\textcolor{green}{-1.35}) & 0.707 (\textcolor{green}{-0.068})\\
        IN+MHA+IN+CA & 40.67 (\textcolor{red}{+1.56})& 135.90 (\textcolor{red}{+3.00})& 29.04 (\textcolor{green}{-0.92}) & 0.699 (\textcolor{green}{-0.076})\\
        Ada+MHA+Ada+SA & 39.36 (\textcolor{red}{+0.26})& 133.82 (\textcolor{red}{+0.92})& 29.13 (\textcolor{green}{-0.83}) & 0.752 (\textcolor{green}{-0.023})\\
        Ada+MHA+Ada+CA (Ours) & \textbf{39.11} & \textbf{132.90} & \textbf{29.96} & \textbf{0.775}\\
        \hline
    \end{tabular}
    \label{tab:attn}
\end{table*}

\begin{table*}[t]
    \centering
    \caption{\textbf{Ablation Study on Background Inpainting.} We compare our MVAnimate with the condition without background inpainting on the TikTok dataset. The best metrics are marked in \textbf{bold}. }
    \begin{tabular}{c|cccc}
        \hline
        Background Inpainting & FID-VID $\downarrow$ & FVD $\downarrow$ & PSNR $\uparrow$ & SSIM $\uparrow$\\
        \hline
        w & \textbf{39.11} & \textbf{132.90} & \textbf{29.96} & \textbf{0.775} \\
        w/o & 39.15 (\textcolor{red}{+0.04}) & 133.67 (\textcolor{red}{+0.77}) & 29.54 (\textcolor{green}{-0.42}) & 0.759 (\textcolor{green}{-0.016})\\
        \hline
    \end{tabular}
    \label{tab:bg}
\end{table*}

\subsection{Ablation Study}

In this section, we prove the effectiveness of the design of our algorithm by implementing several experiments for comparison. We also explain the phenomena in the experimental performances. Note that we place more ablation studies in the supplemented article and video files.

\subsubsection{Multi-View Optimization Loss Design}

The 3 loss functions in the multi-view optimization module are largely based on the previous studies on multi-view shape or 3D model reconstruction. Therefore, we carry out experimental comparisons on these 3 sub-losses. As shown in Tab.~\ref{tab:loss}, our loss design shows a more promising result on the test set of the TikTok dancing dataset than the other incomplete loss expressions. 

\subsubsection{Multi-View Attention Evaluation}
\label{att}

We provide the ablation study results of our Multi-View Attention block, where different attention blocks are compared altogether. 

In Tab.~\ref{tab:attn}, we show the results with different MV-Attention layers implemented in the MVAnimate backbone. $LN$, $IN$, $Ada$ each represents Linear Normalization, Instance Normalization, and AdaLIN~\cite{adalin}. $MHA$, $SA$, $CA$ stand for Multi-Head Attention, Self-Attention, and Cross Attention. The results show that the setting we implement for MV-Attention gains an advantage over other attention policies.

\subsubsection{Background Inpainting}

We implement the background inpainting operation because of the fact that most animation video generation algorithms~\cite{hu2024animate,zhang2024mimicmotion} tend to blur the background and create a changing background texture throughout the video. In this case, the background of the generated video is unified to a certain texture. Besides, the application of camera motion transfer to the output helps the generation result stick to the camera view motion and can adapt to dynamic scenes. We make a quantitative experiment in Tab.~\ref{tab:bg} to test the performance of background inpainting. We compare the results of our MVAnimate with another pipeline, where the reference image is sent into the ReferenceNet module as a whole, and no further background operation is implemented after the MV-Opt module. The results in Tab.~\ref{tab:bg} show that, although our inpainting method is more similar to a grafted model onto our algorithm, it indeed enhances our overall performance on the dataset. We also make some visualized comparisons in the supplementary materials.

\section{Conclusion}

In this work, we propose MVAnimate, an algorithm to generate high-quality character animation videos. Our method addresses existing problems of 2D or 3D pose modeling-based animation algorithms, including complex gestures, clothing distortion, texture contamination, etc. We also propose an optimization algorithm for higher-quality multi-view video generation, setting up an example for multi-view character animation. We hope that our work will inspire other multi-view-based algorithms in character animation or other related research areas. 




\bibliographystyle{unsrt}
\bibliography{reference}






\clearpage

\begin{figure*}[t]
    \includegraphics[width=\textwidth]{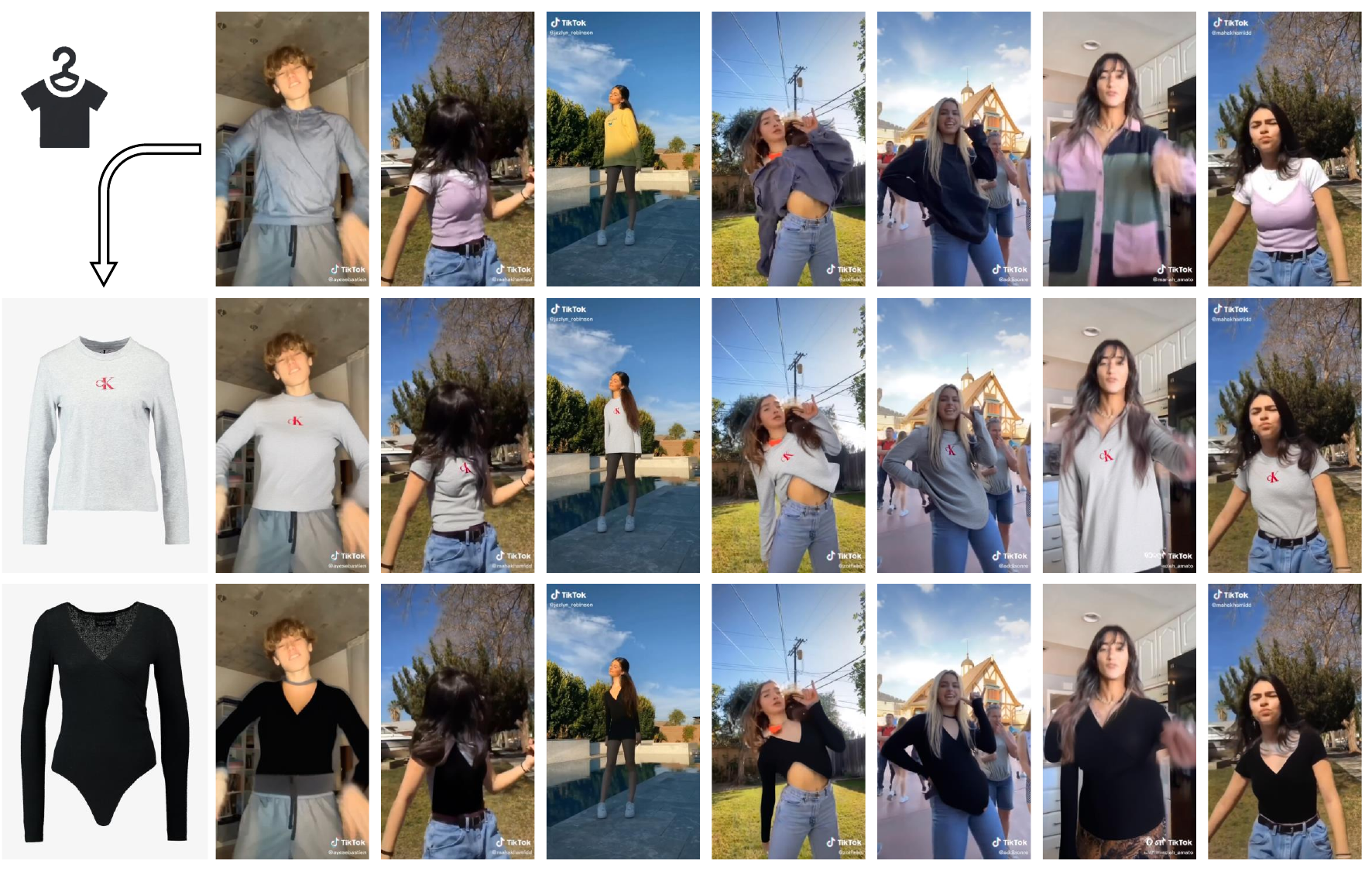}
    \caption{\textbf{Some Examples of Augmentation Results.} Here we show some examples (the second and third rows) augmented from the original TikTok dataset (the first row). These data make up the AugTik we implement as the training dataset. }
    \label{fig:vto}
\end{figure*}%



In our supplementary material, we make a more thorough analysis of our algorithm design, provide more details about our dataset augmentation methods, show more qualitative results to evaluate the background inpainting operation, introduce more experimental results on our MV-Opt module, and give more discussions on the limitations of our MVAnimate. Some more video examples are also attached to the supplementary material. 

\section{Extended Methodology}

In this appendix section, we give a detailed explanation of the design of our MVAnimate, including the comparison with some previous SOTA methods, and some more theoretical or experimental analysis.

\subsection{SVD Backbone Design}

As stated in the main paper, there are numerous algorithms for character animation achieving promising results. Among all these methods, our MVAnimate shows various differences from other SVD-based algorithms. 

\textbf{Modal Selection}: For most previous methods, the pose sequence extracted from the reference video is operated in an encoder and then directed into the denoising U-Net. Numerous methods only extract one type of pose from the video, such as MimicMotion~\cite{zhang2024mimicmotion}, MagicAnimate~\cite{xu2024magicanimate}, LHM~\cite{qiu2025lhm}, etc. There are only a few of them, such as VividPose~\cite{wang2024vividpose}, which implements two pose detection methods, including DWPose and SMPL-X, to separately encode and add up to the noising latent. However, SMPL-X only estimates the character with the parameters of pose $\theta$, shape $\beta$, and facial expression $\phi$, without a parameterization of clothing. Although there are some other 3D human models that fix this problem based on SMPL, such as Cloth3D~\cite{bertiche2020cloth3d} or DeClotH~\cite{nam2025decloth}, they usually build a new model for cloth meshes and introduce a much more time-consuming reconstruction model. Considering that we only need the cloth features from certain views to generate the cloth status from the reference image, a multi-view 2D prior is more suitable for our setting.

\textbf{Pose Guidance}: In our MVAnimate, we choose to insert the pose guidance module into every layer of the denoising U-Net to enhance the effectiveness of human pose guidance in the video generation process. This multi-layer integration allows the pose signal to influence the generation hierarchically across different spatial resolutions, ensuring that both coarse structural layout and fine-grained motion details are consistently aligned with the target pose throughout the denoising process. By distributing the guidance across the network, our approach mitigates the risk of pose information being diluted or lost in deeper layers~\cite{chang2023magicpose,zhang2024mimicmotion}, which can occur when conditioning is applied only once at the input. This helps our method achieve stronger pose fidelity and more temporally coherent motion, particularly in complex or high-frequency motion scenarios.

\subsection{Input View Number}

The input number is a predetermined hyper-parameter, and we set it to 8, considering the neat geometric relation we can obtain. We implement an experiment with different view numbers. Note that when the view number is 1, the MV Attention block is degraded to a zero-tensor generator, distributing nothing to the diffusion backbone. 

\begin{table}[h]
    \centering
    \begin{tabular}{c|cccc}
        \hline
        Num & FID $\downarrow$ & FVD $\downarrow$ & PSNR $\uparrow$ & SSIM $\uparrow$\\
        \hline
        1 & 46.93 & 159.67 & 28.62 & 0.651\\
        2 & 41.50 & 149.61 & 28.82 & 0.720\\
        3 & 42.26 & 157.40 & 28.73 & 0.721\\
        4 & 40.91 & 137.69 & 29.11 & 0.755\\
        8* & \textbf{39.11} & \textbf{132.90} & \textbf{29.96} & \textbf{0.775}\\
        9 & 40.71 & 146.72 & 29.10 & 0.759\\
        \hline
    \end{tabular}
    \caption{\textbf{Results on TikTok Dataset with Different View Numbers.} We select four main metrics to measure the output quality. * stands for the finally implemented version. The best metrics are marked in \textbf{bold}.}
    \label{tab:view}
\end{table}

The results shown in Tab.~\ref{tab:view} prove our hypothesis and help us find a proper balance at the view number of 8. Moreover, with this view number, the detection of the gesture and the formation of spatial-temporal features can be more reasonable. For instance, our premise that opposite views should have the same pixel distribution fails when the view number is odd. This is one of the outcomes for our operation to add horizontal MSE pose loss from the opposite view, which is only possible when the view number is even. Otherwise, only vertical MSE loss will be available.

\section{Dataset Implementation Detail}
\label{sec:data}

In this appendix section, we give a detailed description of the measures we take to implement our method on the TikTok dataset~\cite{tiktok}, as shown in Fig.~\ref{fig:vto}. 

As mentioned in the main paper, due to data deficiency in the TikTok dataset, data augmentation is necessary to address texture contamination. The TED-talks dataset~\cite{tedtalk}, another implemented dataset in our work, does not need augmentation thanks to its large data quantity. We believe this method can be practical for many situations where data deficiency is the main bottleneck. 

\subsection{VTO-Based Data Augmentation}

As mentioned in the main paper, the training data for character animation tasks requires a large amount of annotation operations. This leads to the dilemma that most related datasets have a relatively small quantity compared to other computer vision or graphics tasks. Therefore, various researchers have tried to offer a proper solution to address this problem. Some researchers seek more data from other sources on the Internet~\cite{zhang2024mimicmotion,wang2024disco,xu2024magicanimate,hu2024animate}. Some others just implement their training and testing on a small-scale dataset~\cite{chang2023magicpose}. 

Considering the quantity of data we utilize is less than that used in other large-model-based methods, we decide to apply data augmentation for our dataset, even if the post-processed data quantity can still hardly match up of them. 

We consider that each video can be applied with a virtual-try-on algorithm so that the output videos will share the same pose sequence and have different image/video content. 

After testing different methods, we apply the IDM-VTON algorithm~\cite{choi2025improving} to complete the augmentation. 

\subsection{TikTok Augmentation}

We make augmentation on the TikTok dataset~\cite{tiktok}, due to the deficiency of data quantity, unlike the TED-talks dataset~\cite{tedtalk}, which has a substantial amount of data available. 

In Fig.~\ref{fig:vto}, we display some examples augmented from the TikTok dataset. These samples show fine performance during the inference process. 

We generate AugTik with only 310 of the original TikTok dataset so that the training will not include the other 30 video clips. Theoretically, for each video clip, when we generate $n$ new video clips with new clothes, there should be $C_{n+1}^2$ video pairs that can be implemented in the training process. However, considering the fact that not all original videos can have high-quality new video clips, we set the number of new video pairs to no more than 10. This will keep the number of new video pairs generated from a certain original one at the same level, preventing some of them from taking too large a percentage in the new AugTik. After augmentation, the AugTik dataset consists of 1,500 training video pairs and is utilized to train our MVAnimate. 

\begin{figure*}
  \includegraphics[width=\textwidth]{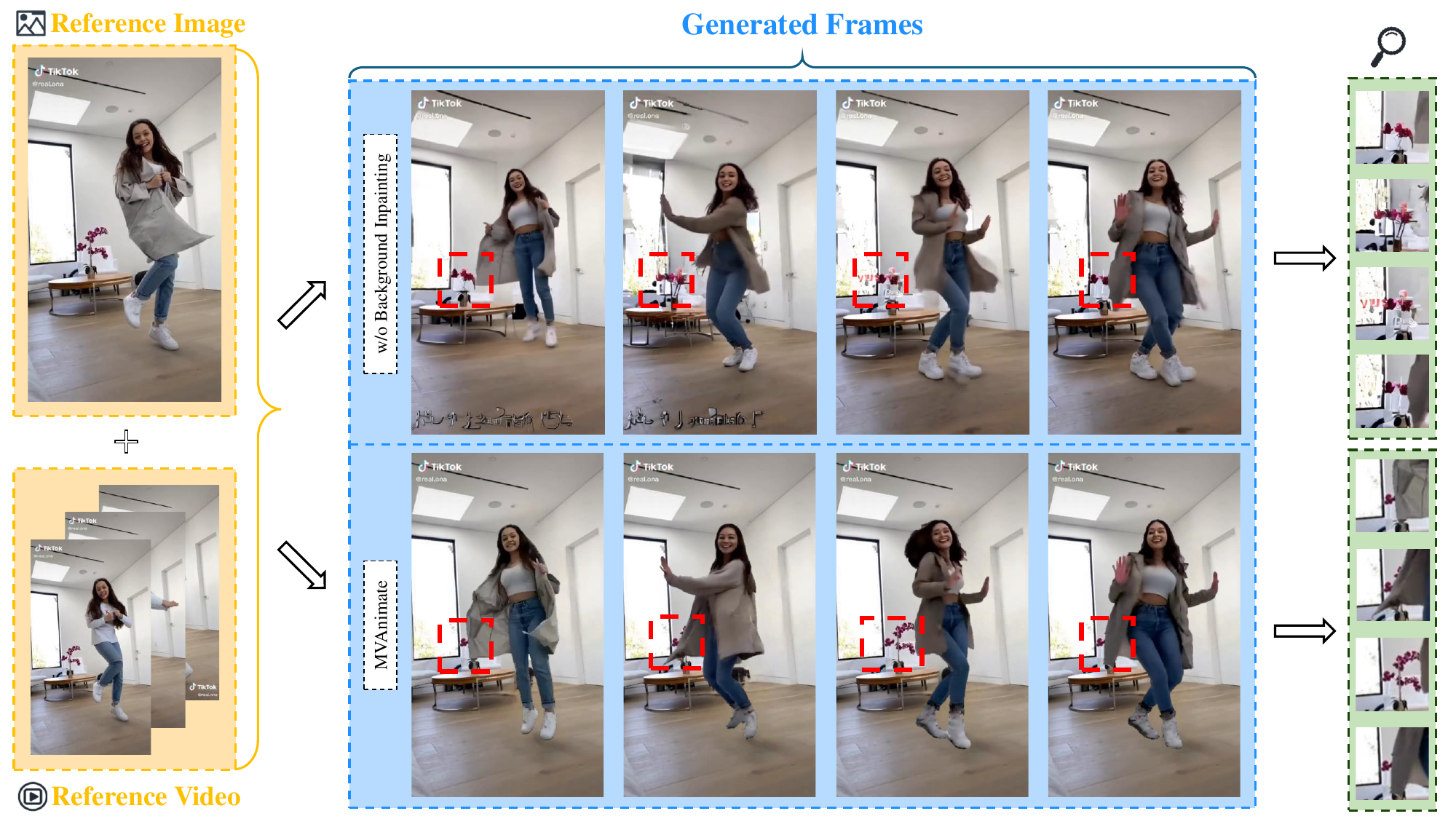}
  \centering
  \caption{\textbf{An Observation on the Effect of Background Inpainting.} We make a comparison without the background inpainting operation, especially on the background texture in several frames. }
  \label{fig:bg}
\end{figure*}

\section{Proof Supplement}

In this section, we provide the supplementary proofs for the propositions and theorems in the main paper. 

\subsection{Proof for Proposition~\ref{prop:invvar}}

\begin{proof}
By substituting $v_t^{(m)} = v_t^\star + \varepsilon_t^{(m)}$ into the estimator,
\[
\hat{v}_t = \sum_{m=1}^M \beta_m (v_t^\star + \varepsilon_t^{(m)})
= \Big(\sum_{m=1}^M \beta_m\Big) v_t^\star + \sum_{m=1}^M \beta_m \varepsilon_t^{(m)}.
\]
Under the constraint $\sum_{m=1}^M \beta_m = 1$, this simplifies to
\[
\hat{v}_t - v_t^\star = \sum_{m=1}^M \beta_m \varepsilon_t^{(m)}.
\]
Hence, the mean squared error is
\[
\mathbb{E}\|\hat{v}_t - v_t^\star\|_2^2
= \mathbb{E}\Big\|\sum_{m=1}^M \beta_m \varepsilon_t^{(m)}\Big\|_2^2.
\]
Because the noises are zero-mean, independent, and isotropic,
\[
\mathbb{E}\Big\|\sum_{m=1}^M \beta_m \varepsilon_t^{(m)}\Big\|_2^2
= \sum_{m=1}^M \beta_m^2 \, \mathbb{E}\|\varepsilon_t^{(m)}\|_2^2
\]
\[
= \sum_{m=1}^M \beta_m^2 \, \mathbb{E}\| \mathcal{N}(0,\sigma_m^2 I_d)\|_2^2.
\]
For $\varepsilon \sim \mathcal{N}(0,\sigma_m^2 I_d)$ we have
$\mathbb{E}\|\varepsilon\|_2^2 = \mathrm{trace}(\sigma_m^2 I_d) = d \sigma_m^2$.
Thus,
\[
\mathbb{E}\|\hat{v}_t - v_t^\star\|_2^2 = d \sum_{m=1}^M \beta_m^2 \sigma_m^2.
\]
Since $d>0$ is a constant factor, minimizing the MSE is equivalent to minimizing
\[
J(\beta) = \sum_{m=1}^M \beta_m^2 \sigma_m^2
\quad \text{subject to} \quad \sum_{m=1}^M \beta_m = 1.
\]
This is a convex quadratic program with a linear constraint. Form the Lagrangian
\[
\mathcal{L}(\beta, \lambda) = \sum_{m=1}^M \beta_m^2 \sigma_m^2 + \lambda \Big( \sum_{m=1}^M \beta_m - 1 \Big).
\]
Taking the derivative with respect to $\beta_m$ and setting it to zero gives
\[
\frac{\partial \mathcal{L}}{\partial \beta_m} = 2 \sigma_m^2 \beta_m + \lambda = 0
\quad \Rightarrow \quad
\beta_m = - \frac{\lambda}{2 \sigma_m^2}.
\]
Plugging this into the constraint $\sum_{m=1}^M \beta_m = 1$ yields
\[
\sum_{m=1}^M \Big( - \frac{\lambda}{2 \sigma_m^2} \Big) = 1
\quad \Rightarrow \quad
- \frac{\lambda}{2} \sum_{m=1}^M \sigma_m^{-2} = 1
\]
\[
\quad \Rightarrow \quad
\lambda = - \frac{2}{\sum_{m=1}^M \sigma_m^{-2}}.
\]
Therefore,
\[
\beta_m = \frac{1}{\sigma_m^2} \cdot \frac{1}{\sum_{j=1}^M \sigma_j^{-2}}
= \frac{\sigma_m^{-2}}{\sum_{j=1}^M \sigma_j^{-2}}.
\]
This shows that the MSE-minimizing linear estimator assigns weights inversely proportional to the per-view noise variances, completing the proof.
\end{proof}

\subsection{Proof for Proposition~\ref{prop:convex}}

\begin{proof}
Each term is a convex quadratic with Hessian $(I-A)^\top(I-A)\succeq 0$ when $\|A\|_2\le 1$.
The sum is convex; averaged operators with coefficients summing to one and nonexpansive summands remain nonexpansive.
\end{proof}

\subsection{Proof for Theorem~\ref{thm:monotone}}

\begin{proof}
By construction of $F$, each term $\mathcal{L}_{\text{temp}}, \mathcal{L}_{\text{mvP}}, \mathcal{L}_{\text{mvS}}$ is a finite sum of squared or $\ell_1$ distances between frames/views (see Eq.~\ref{eq:loss}), hence $F \ge 0$ and $F$ is bounded below. Moreover, the MV-Opt schedule updates one block (either a frame or a view inside a frame) at a time while keeping all other blocks fixed. By the sufficient-decrease assumption, every block update strictly decreases $F$ unless the current block is already optimal, so $\{F(x^{k})\}$ is a monotonically nonincreasing sequence bounded below, hence it converges.

Since $F$ has bounded level sets, the generated sequence $\{x^{k}\}$ is bounded and thus admits limit points. Standard results on cyclic block coordinate descent for continuously differentiable objectives with sufficient decrease~\cite{tseng2001convergence} then imply that every limit point is block-coordinate stationary, i.e., for each block there is no feasible descent direction when all other blocks are fixed. 
\end{proof}

\section{Ablation Study Supplement}

\subsection{Background Inpainting Evaluation}
\label{sec:bg}

In Fig.~\ref{fig:bg}, we compare the results with our method without background inpainting, which did not exclude or inpaint the background during the generation process. We also explicate the details of the preprocessing on the right by magnifying the red-box regions to make the comparison more obvious. 

The result indicates that the generated video clip with background inpainting shows a more stable background, making it closer to reality, which is the same as we expected. Considering that the inpainting process also quantitatively improves the method, we believe it is effective overall. We also place some dynamic experimental results in the video attached to the supplementary material. 

Moreover, implementing background inpainting on other SOTA character animation method~\cite{zhang2024mimicmotion} can also stabilize the background and benefit the generation process. 

\subsection{Multi-View Optimization Evaluation}
\label{sec:mv}


In this section, we show a figure to help illustrate our proposal to implement different weight values on the novel views. During the optimization process of the main view, novel views should be assigned with different weights, as shown in Fig.~\ref{fig:mv}. This also explains why the novel views from the back are less affected in the MV-Opt module. 

\begin{figure}[h]
    \centering
  \includegraphics[width=0.33\textwidth]{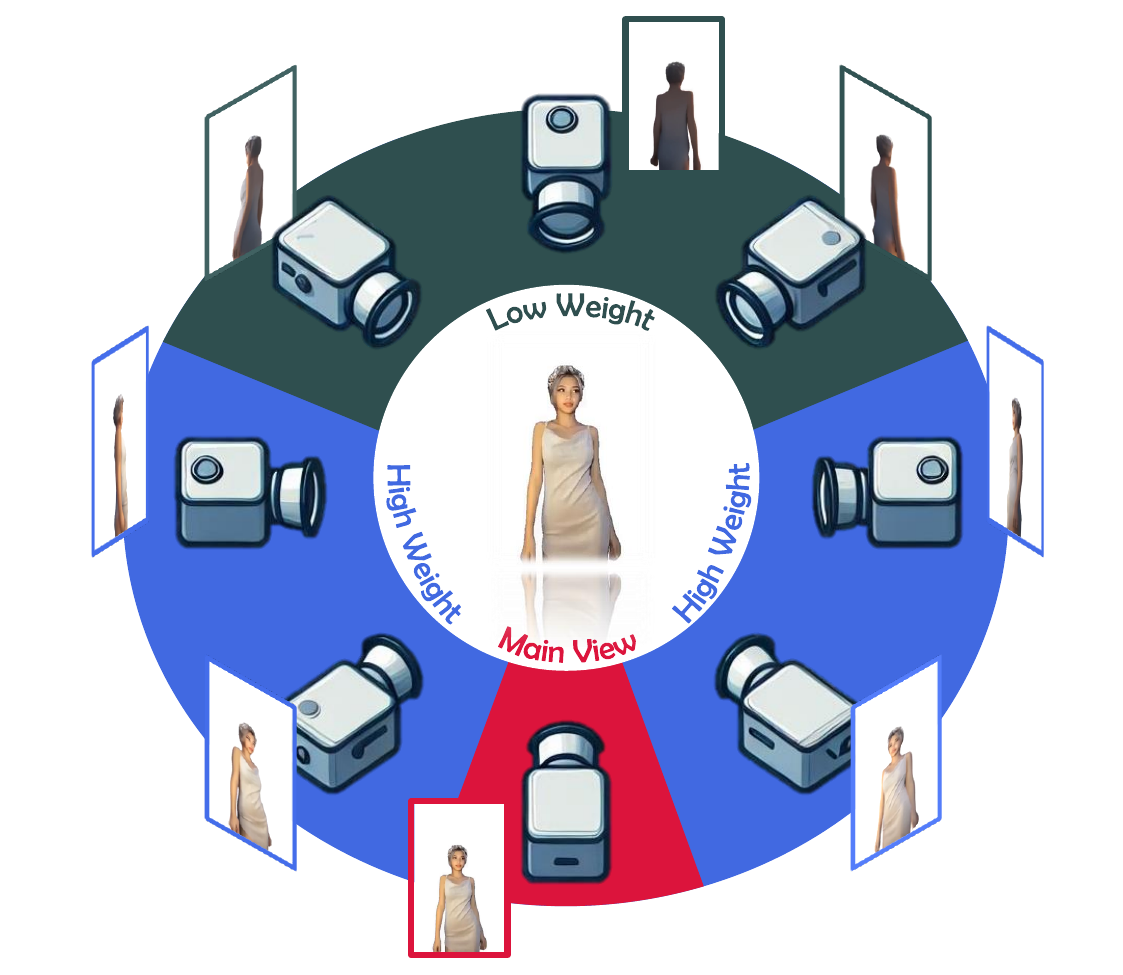}
  \caption{\textbf{Multi-View Weight Distribution.} The novel views closer to the main view have higher weights during training and are more impacted in return. For novel views in the back, it is the opposite situation. }
  \label{fig:mv}
\end{figure}

\section{Discussions}
\label{sec:dis}

\subsection{Corner Case Performance}

As mentioned in the submitted main body, there are two common types of failure in character animation algorithms' performances, including complex poses and texture contamination. 

\begin{figure}[h]
\centering
  \includegraphics[width=0.5\textwidth]{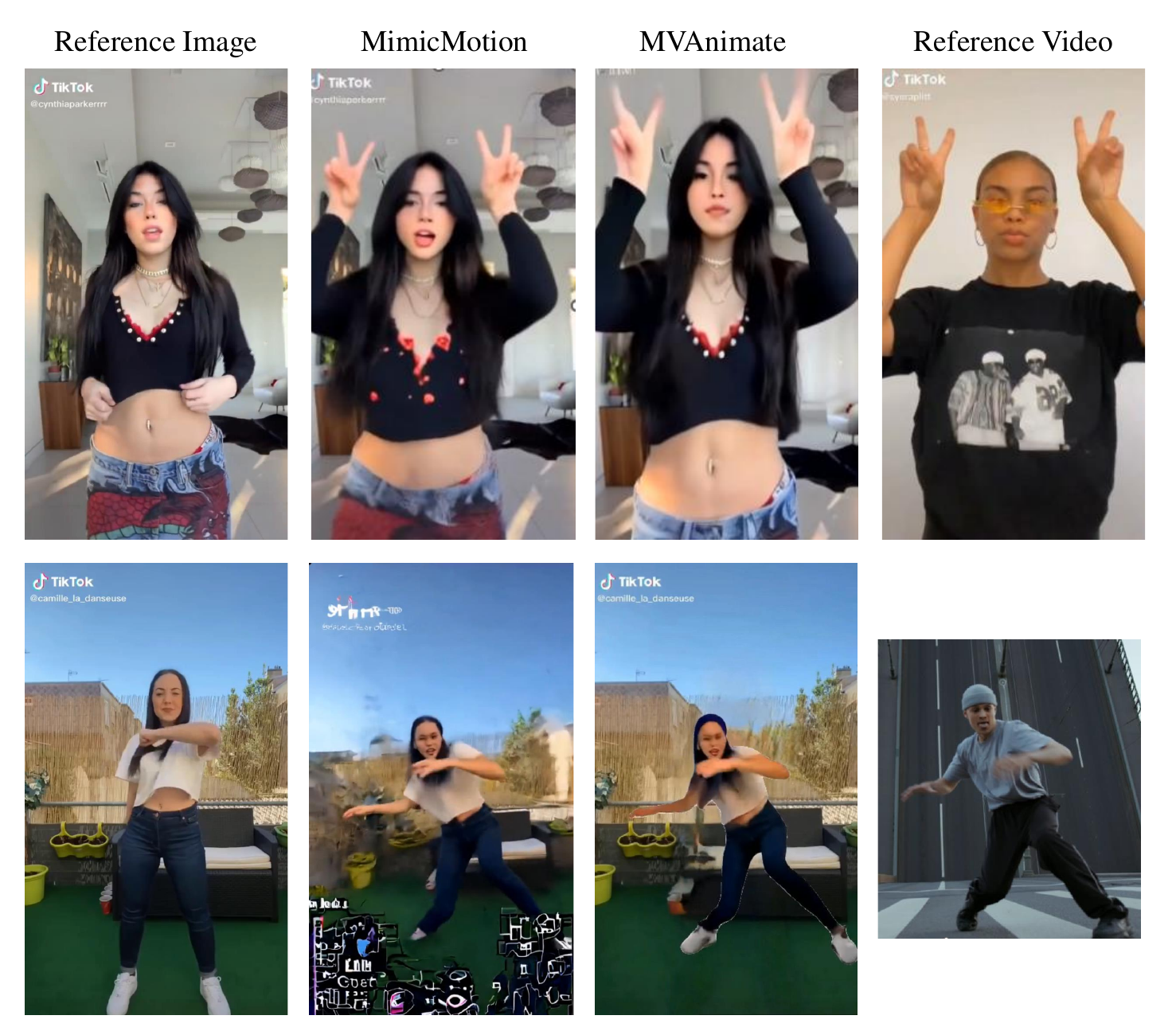}
  \centering
  \caption{\textbf{Corner Cases.} We display some common corner cases in character animation tasks, including texture contamination and complex poses. }
  \label{fig:corner}
\end{figure}

Complex poses with twists or turns can greatly harm the performance of the character animation model. Our MVAnimate can reduce this damage to some extent with the guidance of multi-view video inputs. 

As shown in the first example of Fig.~\ref{fig:corner}, we display some examples of the inference results of our MVAnimate addressing the problem of texture contamination. The output of MVAnimate shows a unified dressing texture throughout the whole output video. 

We also mention the problem of texture contamination in several cases, and we introduce the improvement of our model in these corner cases. As shown in the second example in Fig.~\ref{fig:corner}, our MVAnimate generates a clear motion gesture for the inference image. Although our model still has flaws in some frames or video clips due to the absence of this type of training sample, it still shows a relatively high enhancement compared with that of other methods. 

Some more corner cases are also shown in the video supplementary material. 

\subsection{Limitations}

Although we obtain 3D information with the pre-trained models' multi-view videos, certain problems still remain unsolved. 

\textbf{Multiple Characters.} Our MVAnimate is only designed for a one-character scenario, which means that the Openpose-based multi-view guidance network would be misled, causing MVAnimate to be unable to generate new animation videos. More methods will be introduced to address such problems in the future. 

\textbf{Dynamic Scenes.} We implement some tricks, such as background inpainting, which might be degraded for dynamic or outdoor scenes~\cite{li2025efficient,liu2024dynvideo}. The TikTok dataset~\cite{tiktok} and the TED-talks dataset~\cite{tedtalk} only contain a few dynamic scenes, making this problem less severe in our experiments. However, these problems will surely make certain other scenarios tough for our MVAnimate. 

\end{document}